\documentclass{fairmeta}
\usepackage{amsmath,amssymb,amsthm}
\usepackage{booktabs}
\usepackage{array}
\usepackage{graphicx}
\usepackage{hyperref}
\usepackage{url}

\newcommand{\E}{\mathbb E}
\newcommand{\KL}{D_{\mathrm{KL}}}
\newcommand{\dd}{\mathrm d}
\newcommand{\Var}{\operatorname{Var}}
\newcommand{\Cov}{\operatorname{Cov}}
\newcommand{\R}{\mathbb R}
\newtheorem{proposition}{Proposition}
\newtheorem{theorem}{Theorem}

\title{Equal Path Cost, Unequal Output Effects: Understanding Perturbation Propagation in Diffusion Models}
\author[1,*]{Wei Guo}
\author[2,*]{Yaowen Zhang}
\author[1]{Xingtong Ge}
\author[1,\dagger]{Jun Zhang}
\affiliation[1]{iComAI Lab, The Hong Kong University of Science and Technology}
\affiliation[2]{Department of Physics, The Hong Kong University of Science and Technology}
\contribution{\textsuperscript{*}Equal contribution. \quad \textsuperscript{\textdagger}Corresponding author.}
\hypersetup{pdfauthor={Wei Guo, Yaowen Zhang, Xingtong Ge, and Jun Zhang},
  pdftitle={Equal Path Cost, Unequal Output Effects: Understanding Perturbation Propagation in Diffusion Models}}

\begin{document}
\abstract{
Diffusion models have achieved remarkable success in generative modeling, with their sampling procedures routinely modified to control generation and improve efficiency. These modifications introduce perturbations along the sampling trajectory, raising a central question: how do such perturbations affect generated output? To address this question, we develop a theoretical framework to investigate perturbation propagation, combining dynamical analysis of the sampling process with an information-theoretic characterization of output responses. Within this framework, we quantify perturbation strength using the Kullback--Leibler (KL) divergence between perturbed and reference trajectory distributions, termed as \emph{path cost}, which is shown to bound, but do not determine, changes in the output distribution. Building on this analysis, we derive a response identity that connects the propagation and accumulation of local perturbations with the information captured by a selected feature mean, explaining why changes in the output distribution can remain undetected by its first-order response. We test our theoretical analysis through controlled interventions at equal path cost in pretrained diffusion models, revealing distinct patterns of output sensitivity across sampling stages and spatial frequencies. To assess whether our framework can diagnose perturbations arising from practical approximations, we apply it to cache-based acceleration and show that our propagation analysis reliably identifies sampling intervals where caching causes larger image errors.
}

\maketitle
\section{Introduction}
\label{sec:introduction}
Diffusion models generate samples by simulating learned reverse-time dynamics from a noise distribution to the data distribution \citep{ho2020denoising,song2021score}. During inference, these dynamics are perturbed by errors in learned score estimates or denoising predictions \citep{song2021maximum,li2024error}, by deliberate guidance updates \citep{dhariwal2021diffusion}, and by approximations introduced by fast solvers, quantization, and feature caching \citep{lu2022dpmsolver,li2023qdiffusion,ma2024deepcache}. Because a perturbation changes the input to subsequent denoising steps, its effect on the generated sample depends on the downstream sampling dynamics, not only on its magnitude at the timestep where it occurs \citep{li2024error,liu2026epacache}. This creates a mismatch between the trajectory-level quantities used to characterize a perturbation and the output-level quantities used to evaluate its consequences: the former describe the change where it enters the sampling process, whereas the latter concern the endpoint distribution, the response of a selected feature, or the displacement of a generated sample. This mismatch raises a central question: \emph{How do perturbations along the sampling trajectory translate into changes in generated outputs, and to what extent can trajectory-level measures characterize those changes?}

Existing work has addressed different aspects of this question. Analyses of score estimation and discretization use aggregated errors along the sampling trajectory to bound deviations in the generated distribution \citep{song2021maximum,chen2023sampling,benton2024nearly}. Studies of exposure bias and error propagation instead examine how prediction errors accumulate across denoising steps \citep{ning2023input,ning2024exposure,li2024error}. More recently, in the setting of recursively trained diffusion models, \citet{khelifa2026quantifying} derive terminal-divergence lower bounds parameterized by an endpoint-observability coefficient, which measures how much pathwise score-error energy remains detectable in the terminal distribution. These results show that inaccuracies along the sampling trajectory can accumulate into the global generation error and the pathwise error may be only partially observable in the terminal distribution. However, they do not fully characterize how the downstream sampling dynamics mediate the relationship between the magnitude of a perturbation along the trajectory and its effect on the final generated output.

To address this gap, we develop a theoretic framework combining information-theoretic analysis with a dynamical characterization of output response, which explains why perturbations of equal strength can produce unequal output effects and derives how local perturbations propagate and accumulate to determine output responses. We measure perturbation strength by the KL divergence
between perturbed and reference trajectory distributions, which we call \emph{path cost}.
The analysis shows that a perturbation’s effects on generated outputs depend jointly on its path cost, how it propagates and interacts with perturbations at other steps, and how the resulting changes are measured. Our contributions are:
\begin{itemize}
    \item {\textbf{A theoretical framework explaining how perturbations affect generated outputs.}} We derive an exact decomposition of path cost that distinguishes changes in the output distribution from changes in trajectory distributions conditional on the output. We derive a first-order propagation law that quantifies how subsequent sampling dynamics transform local perturbations and how their contributions accumulate to determine the final output response. We further establish bounds relating feature mean responses to local output-distribution changes, explaining why a selected feature mean can fail to detect such changes. These results explain why equal path costs can produce unequal output effects.
    \item {\textbf{Experimental support for the theoretical analysis.}} We test the framework through controlled perturbations at equal path cost, varying the sampling stages at which they are applied and their spatial directions. The resulting output responses differ across sampling stages and directions. The experiments show that first-order propagation can capture the temporal patterns of responses to finite perturbations, although predicted and observed magnitudes can differ substantially. They also show that smaller absolute feature mean responses can correspond to larger changes relative to baseline fluctuations, consistent with our theoretical analysis.
    \item {\textbf{Applying the framework to diagnose output errors in cache-based acceleration.}} We apply the framework to cache-based acceleration for diffusion generation, where earlier model predictions are reused to avoid fresh evaluations. By tracing how the resulting approximation errors propagate and accumulate, we assess their impact on the final image and identify sampling intervals where caching causes larger output deviations. This propagation-based diagnostic identifies such intervals more reliably than local error magnitude and path-cost measures, which remains informative even when predicted and observed error magnitudes differ substantially.
\end{itemize}

% Specifically, we use path KL divergence and its local Fisher limit to place perturbations at different noise levels on a common, noise-aware scale \citep{berner2024control,dupuis2016path}. This scale quantifies how strongly the full sampling trajectory is changed, but does not specify how much of that change remains visible at the endpoint. We therefore condition the path-space information on the endpoint, separating changes in the output distribution from changes confined to conditional trajectory histories \citep{leonard2014survey,pollard2011fisher}. Because endpoint-visible information does not by itself determine the response of a particular measurement, we further relate it to the mean response of a chosen observable on its baseline fluctuation scale \citep{dupuis2016path,dechant2020fluctuation}. Finally, differentiating through the sampler expresses this response as a sum of transported local contributions, making explicit how downstream dynamics and cross-time alignment shape the effect observed at the output \citep{mohamed2020monte}. 
\section{Related work}
\label{sec:related_work}

\paragraph{Error propagation in diffusion sampling.}
Path-space control relates diffusion modeling to KL minimization \citep{berner2024control}, while score-error analyses provide likelihood and sampling-error bounds \citep{song2021maximum,chen2023sampling,benton2024nearly}. \citet{li2024error} relate local and cumulative errors through a propagation equation. \citet{khelifa2026quantifying} quantify the endpoint visibility of pathwise score errors to derive terminal-divergence lower bounds, with applications to recursive training. We study perturbations of a fixed sampler, connecting output visibility to downstream propagation and the information captured by a selected feature mean.

\paragraph{Conditional information and observable response.}
Conditional-score projections characterize Fisher information retained by an observation \citep{pollard2011fisher}, while fluctuation--response bounds relate mean sensitivity to variance and path information \citep{dupuis2016path,dechant2020fluctuation}. \citet{debortoli2025dimensionfree} bound diffusion sampling errors using smooth test functions. We build on these information and response identities to characterize how much path information reaches the output and how much a feature mean captures, with propagation determining its response.

\paragraph{Spectral generation dynamics.}
\citet{benita2025spectral} derive Gaussian spectral transfer dynamics and use data spectra to design schedules. We compare temporal and directional responses at matched path cost, and distinguish absolute response from response relative to baseline feature fluctuations.

\paragraph{Caching and sensitivity-based acceleration.}
DeepCache reuses high-level features \citep{ma2024deepcache}; SenCache uses denoiser sensitivity to latent and timestep changes \citep{haghighi2026sencache}. ERTACache analyzes and corrects feature-shift and step-amplification errors \citep{peng2026ertacache}, while EpaCache calibrates reuse thresholds from terminal deviations under single-step caching interventions \citep{liu2026epacache}. We use reference-trajectory propagation to diagnose image errors from multi-step output caching and test whether it identifies sampling windows with larger errors more reliably than local-defect and path-cost measures.

\section{A Theoretical Framework for Perturbation Propagation}
\label{sec:path_space}
\label{sec:endpoint}

In this section, we derive a theoretical characterization of how local perturbations propagate through diffusion sampling and affect measured output responses. We first quantify trajectory changes by path cost, separate output changes from conditional trajectory changes, and derive perturbation propagation. Our main theorem expresses the measured response through the propagated local perturbations and identifies the fractions of path information retained in the output distribution and captured by the chosen feature mean.

%Together, these results explain how downstream dynamics, cross-time accumulation, and observation shape perturbation propagation. We give the proofs and regularity conditions in Appendix~\ref{app:path_information}.

\subsection{How path cost quantifies perturbation strength}
\label{sec:path_cost}

Diffusion models learn to reverse a forward process that progressively adds noise to data. Generation starts from noise and follows learned reverse dynamics through discrete sampling updates \citep{ho2020denoising,song2021score}. Common stochastic diffusion samplers update the state using a model-dependent mean and Gaussian noise \citep{nichol2021improved}. These updates can then be expressed in the following form
\begin{equation}
    X_{i+1}
    = \mathcal G_i^0(X_i,\xi_i)
    = \mu_i(X_i)+L_i(X_i)\xi_i,
    \qquad
    \xi_i\overset{\mathrm{i.i.d.}}{\sim}\mathcal N(0,I).
    \label{eq:reference_sampler}
\end{equation}
Here, $X_i$ denotes the sampling state after $i$ generation steps, with $X_0$ drawn from the initial noise distribution and $X_M$ denoting the terminal sampling state. The index $i=0,\ldots,M-1$ follows generation order from noise toward the output. The map $\mathcal G_i^0$ defines one update of the reference sampler with the superscript $0$ indicating the unperturbed process.The variables $\xi_i$ are independent standard Gaussian noise vectors, and $L_i(X_i)\xi_i$ is the stochastic component of the sampling update, with conditional covariance $\Lambda_i(X_i)=L_i(X_i)L_i(X_i)^\top$.

As mentioned above, local perturbations, such as prediction errors and computational approximations, modify the sampling updates in Eq.~\eqref{eq:reference_sampler}. We model the resulting dynamics by a family of perturbed sampling processes parameterized by perturbation strength $\varepsilon$, with $\varepsilon=0$ corresponding to the reference process. To study how local perturbations affect the sampling process, we consider the complete trajectories $\Gamma=(X_0,\ldots,X_M)$ and denote their joint distribution under the reference and perturbed processes as $P$ and $Q_\varepsilon$, respectively, so that $Q_0=P$. Such effects can be measured through the statistical distinguishability of these trajectory distributions using their Kullback--Leibler (KL) divergence \citep{cover2006elements}, which we define as the \emph{path cost}, $\mathcal C(\varepsilon)=\KL(Q_\varepsilon\|P)$.

We consider perturbations that preserve the initial distribution and shift each Gaussian transition mean by $\varepsilon u_i$, while retaining its positive definite covariance $\Lambda_i$. The resulting transition densities are strictly positive, ensuring that $Q_\varepsilon$ is absolutely continuous with respect to $P$, i.e., $Q_\varepsilon\ll P$ \citep{kallenberg2021foundations}. Then, by using the relative-entropy chain rule \citep{cover2006elements}, we can obtain the expression of the path cost as a sum of conditional transition KLs as
\begin{equation}
 \mathcal C(\varepsilon):=\KL(Q_\varepsilon\|P)
 =\frac{\varepsilon^2}{2}\E_{Q_\varepsilon}
 \sum_i u_i^\top\Lambda_i^{-1}u_i.
 \label{eq:path_cost}
\end{equation}
This expression shows how local mean perturbations contribute to the change in the trajectory distribution. Each contribution is weighted by the inverse transition covariance: for $\Lambda_i=\nu_i I$, the conditional KL at step $i$ is $\varepsilon^2\|u_i\|^2/(2\nu_i)$. Thus, the same mean shift produces a larger change in the transition distribution, as measured by KL divergence, when the sampling noise is lower. The derivation details are provided in Appendix~\ref{app:discrete_cost}. 

\subsection{Why path cost does not determine output changes}
\label{sec:conditional_split}
\begin{figure}[t]
 \centering
 \includegraphics[width=\linewidth]{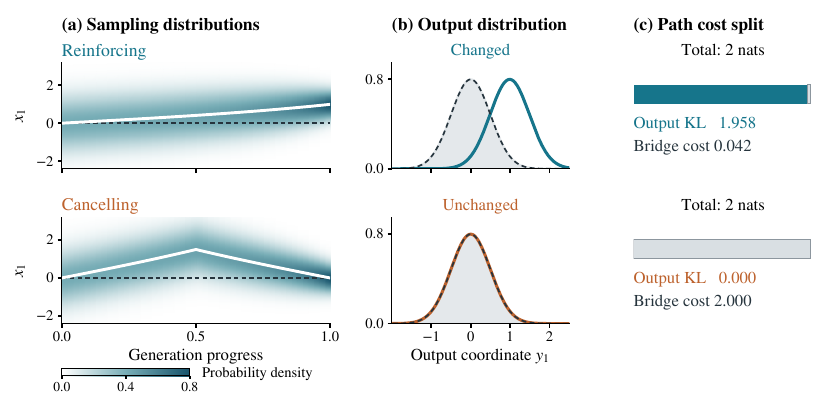}
 \caption{\textbf{Equal path costs, different output-distribution changes.} Two temporal perturbation sequences act only on the first coordinate of a two-dimensional Gaussian reverse diffusion, each with path cost $\mathcal C=2$ nats. Their contributions reinforce (top) or cancel (bottom) at the output. (a) Coordinate density evolution on a shared color scale, with perturbed (white) and reference (dashed) means. (b) Perturbed (solid) and reference (dashed) output densities. (c) Output KL (colored) and bridge cost (gray). Construction details are in Appendix~\ref{app:toy_extensions}.}
 \label{fig:toy_visibility}
\end{figure}

Having quantified how perturbations change the sampling trajectory distribution, we now examine their effect on the distribution of generated outputs. The terminal sampling state $X_M$ may require a final deterministic transformation, such as decoding a latent representation, to produce the generated output. We therefore write $Y=\phi(X_M)$, where $\phi$ denotes this transformation and is the identity when $X_M$ is already the output. We use the same map for both the reference and perturbed processes. 

Then, we derive an exact decomposition of path cost into the change in the output distribution and the change in trajectory distributions conditional on the output in the following Proposition \ref{prop:kl_split}. 
\begin{proposition}[Decomposition of path and output changes]
\label{prop:kl_split}
For $Q_\varepsilon\ll P$ with finite path KL and regular conditional path distributions, let $p_Y$ and $q_{\varepsilon,Y}$ be the output distributions under $P$ and $Q_\varepsilon$. Then, the path cost can be decomposed as 
\begin{equation}
 \mathcal C(\varepsilon)
 =\underbrace{\KL(q_{\varepsilon,Y}\|p_Y)}_{\mathcal C_Y(\varepsilon)}
 +\underbrace{\E_{q_{\varepsilon,Y}}\KL\!\left(Q_\varepsilon(\cdot\mid Y)\|P(\cdot\mid Y)\right)}_{\mathcal C_{\mathrm{bridge}\mid Y}(\varepsilon)}.
 \label{eq:kl_split}
\end{equation}
\end{proposition}
\begin{proof}
    Please see Appendix~\ref{app:path_decomposition}.
\end{proof}
The above decomposition separates two consequences of a perturbation. The first term, $\mathcal C_Y(\varepsilon)$, measures the change distinguishable from generated outputs. The second measures how the trajectory distributions differ when conditioned on the same output, averaged over the perturbed output distribution, which we call as the \emph{bridge cost}. Since bridge cost is non-negative, we have $0\le\mathcal C_Y(\varepsilon)\le\mathcal C(\varepsilon)$. Therefore, the path cost bounds the changes in output distribution, but equal path costs can yield different output changes depending on their bridge components. 

To illustrate this decomposition, Figure~\ref{fig:toy_visibility} presents a Gaussian toy
example that we construct to show how perturbations with equal path costs can produce different output effects. In this example, scaling the cancelling perturbation sequence yields any positive path cost while output KL remains zero. Thus, a positive path cost alone does not imply a positive output-distribution change. To understand how perturbations from different steps reinforce or cancel at the output, we next examine their propagation through the sampling dynamics.

\subsection{How local perturbations propagate and accumulate}
\label{sec:response_propagation}

To trace how local perturbations affect generated outputs, we analyze the first-order response with respect to perturbation strength at $\varepsilon=0$. This gives a linear description of propagation along the reference trajectories.

Let $X_{i+1}^\varepsilon=\mathcal G_i^\varepsilon(X_i^\varepsilon,\xi_i)$ denote the perturbed update, with $\mathcal G_i^\varepsilon|_{\varepsilon=0}=\mathcal G_i^0$ as in Eq.~\eqref{eq:reference_sampler}. To compare corresponding trajectories, we hold the initial state and noise variables fixed. Along each reference trajectory, we define $J_i=D_x\mathcal G_i^0$ as the Jacobian matrix of the reference update with respect to its input state, which describes how an input-state change propagates to the next step. We also define $e_i=\left.\partial_\varepsilon\mathcal G_i^\varepsilon\right|_0$, as the derivative of the sampling update with respect to perturbation strength with its input state and noise remained fixed. This term measures the local perturbation introduced by the update itself. For the mean perturbations in Section~\ref{sec:path_cost}, $e_i=u_i$.

To examine how local perturbations affect a particular property of the generated output, we represent that property by a scalar feature $f(Y)$, such as a spatial-frequency component of an image\citep{ahmed1974discrete}. Its expectation summarizes this property across generated samples. Therefore, we study how the expected feature value changes under perturbation and define the feature mean response \citep{kubo2012statistical} as $\chi_f=\left.\partial_\varepsilon\E_{Q_\varepsilon}f(Y)\right|_0$. Using pathwise differentiation \citep{mohamed2020monte}, we express both the output response and the feature mean response in the following proposition.

\begin{proposition}[Perturbation propagation dynamics]
\label{prop:transport_response}
Assume that $\mathcal G_i^\varepsilon$ is differentiable in the state and perturbation strength, and that $\phi$ is differentiable at the terminal state, almost surely along reference trajectories. The output response, defined as the derivative of the generated output with respect to perturbation strength at $\varepsilon=0$, satisfies
\begin{equation}
 V:=\left.\partial_\varepsilon Y^\varepsilon\right|_0
 =\sum_{i=0}^{M-1}T_i e_i,
 \qquad T_i=D\phi(X_M)J_{M-1}\cdots J_{i+1},
 \label{eq:full_response_transport}
\end{equation}
where $T_i$ propagates the local perturbation at step $i$ through the remaining sampling updates and the output map $\phi$. Furthermore, if $f$ is continuously differentiable and $[f(Y^\varepsilon)-f(Y)]/\varepsilon$ is dominated by an integrable random variable near $\varepsilon=0$, then the feature mean response is given as
\begin{equation}
 \chi_f=\E_P[\nabla f(Y)^\top V]
       =\sum_{i=0}^{M-1}\E_P[\nabla f(Y)^\top T_i e_i].
 \label{eq:tangent_response}
\end{equation}
\end{proposition}
\begin{proof}
    Please see Appendix~\ref{app:propagation}.
\end{proof}

Proposition~\ref{prop:transport_response} shows how downstream dynamics transform each local perturbation through $T_i$, changing its strength and direction at the output. The propagated perturbations from different steps add to form the total output response $V$. To examine how their magnitudes and relative directions determine the squared magnitude of this response, we write $v_i=T_i e_i$ and expand the norm of $V$ as
\begin{equation}
 \|V\|^2=\sum_i\|v_i\|^2+2\sum_{i<j}\langle v_i,v_j\rangle.
 \label{eq:response_accumulation}
\end{equation}
The first term sums the squared magnitudes of the propagated perturbations considered separately. The cross terms account for their relative directions: positive inner products increase the total squared response, while negative inner products reduce it. Thus, the magnitude of the total response depends on both the sizes of the propagated perturbations and how their directions align at the output.

The above analysis explains how local perturbations propagate and accumulate to determine the output response and the mean response of a selected feature. However, this feature mean describes only one property of the output distribution. Interpreting its response therefore requires understanding how much distributional change it captures and how this depends on the choice of feature.

\subsection{What feature mean responses reveal about output changes}
\label{sec:observable_response}

To identify the distributional changes captured by the feature mean, we derive its response in terms of trajectory likelihoods. Differentiating $\E_{Q_\varepsilon}f(Y)=\E_P[f(Y)\,\dd Q_\varepsilon/\dd P]$ with respect to perturbation strength at $\varepsilon=0$, we obtain
\begin{equation}
 \chi_f=\E_P[f(Y)S],
 \label{eq:path_score_response}
\end{equation}
where $S=\left.\partial_\varepsilon\log(\dd Q_\varepsilon/\dd P)\right|_0$ is the path likelihood score, describing changes in trajectory log likelihood with perturbation strength.

To identify the leading dependence of path cost on perturbation strength, we derive the expansion of Eq.~\eqref{eq:path_cost} near the reference sampler as
\begin{equation}
 \mathcal C(\varepsilon)=\tfrac12\varepsilon^2F_{\mathrm{path}}+o(\varepsilon^2),
 \label{eq:path_fisher}
\end{equation}
where $F_{\mathrm{path}}=\E_P S^2$ is the path Fisher information \citep{amari2016information,dupuis2016path}. Since $f(Y)$ depends only on the output, we can average $S$ over trajectories sharing that output without changing $\chi_f$. This gives the output likelihood score $S_Y=\E_P[S\mid Y]$ \citep{pollard2011fisher}. We define the output Fisher information as $F_Y=\E_P S_Y^2$ and derive the second-order expansion of the KL decomposition in Eq.~\eqref{eq:kl_split}, yielding
\begin{equation}
 F_{\mathrm{path}}=F_Y+\underbrace{\E_P\Var_P(S\mid Y)}_{F_{\mathrm{bridge}\mid Y}}.
 \label{eq:fisher_split}
\end{equation}
The first term determines local output-distribution change through $\mathcal C_Y(\varepsilon)=\tfrac12\varepsilon^2F_Y+o(\varepsilon^2)$, while the second is the local bridge component. For $F_{\mathrm{path}}>0$, the ratio $F_Y/F_{\mathrm{path}}$ therefore gives the fraction of path cost visible at the output as $\varepsilon\to0$. The assumptions and derivations for Eqs.~\eqref{eq:path_fisher} and~\eqref{eq:fisher_split} are provided in Appendix~\ref{app:projection}.

To compare the feature mean response with the output and path Fisher information, we define the squared normalized mean response as $B_f=\chi_f^2/\Var_P(f(Y))$. For small $\varepsilon$, the squared mean shift on this scale is $\varepsilon^2B_f+o(\varepsilon^2)$, matching the perturbation order of the local KL expansions. We derive a lower bound on output Fisher information from this normalized response using the fluctuation--response bound \citep{dupuis2016path, dechant2020fluctuation} in the following Proposition \ref{prop:response}.
\begin{proposition}[Feature response and output information]
\label{prop:response}
Assume that the path likelihood ratio is differentiable in mean square at $\varepsilon=0$, meaning $\E_P[(\dd Q_\varepsilon/\dd P-1-\varepsilon S)^2]=o(\varepsilon^2)$, with $\E_P S^2<\infty$. Then any real-valued output feature $f(Y)$ with $0<\Var_P(f(Y))<\infty$ satisfies
\begin{equation}
 \chi_f=\E_P[f(Y)S]=\E_P[f(Y)S_Y],
 \qquad B_f=\frac{\chi_f^2}{\Var_P(f(Y))}\le F_Y\le F_{\mathrm{path}}.
 \label{eq:response_bound}
\end{equation}
\end{proposition}
\begin{proof}
    Please see Appendix~\ref{app:response}.
\end{proof}
Proposition~\ref{prop:response} identifies two stages at which perturbation information can be lost, from the trajectory to the output distribution and from the output distribution to the selected feature mean. We derive an exact expression for the fraction of path information captured by the chosen feature mean, incorporating both perturbation propagation and the two stages of information loss in the following Theorem \ref{thm:propagation_visibility}.

\begin{theorem}[Perturbation propagation and measured output response]
\label{thm:propagation_visibility}
Under the assumptions of Propositions~\ref{prop:transport_response} and~\ref{prop:response}, let $f$ be continuously differentiable with $0<\Var_P(f(Y))<\infty$. If $F_Y>0$, the fraction of path information captured by the feature mean response satisfies
\begin{equation}
 \frac{B_f}{F_{\mathrm{path}}}
 =\frac{\left(\sum_i\E_P[\nabla f(Y)^\top T_i e_i]\right)^2}{\Var_P(f(Y))\,F_{\mathrm{path}}}
 =\underbrace{\frac{F_Y}{F_{\mathrm{path}}}}_{\text{output visibility}}
 \underbrace{\operatorname{Corr}_P(f(Y),S_Y)^2}_{\text{feature--score alignment}}.
 \label{eq:response_visibility_alignment}
\end{equation}
\end{theorem}
\begin{proof}
    Please see Appendix~\ref{app:response_visibility}.
\end{proof}

Theorem~\ref{thm:propagation_visibility} explains why perturbations with equal path cost can produce different measured responses. The propagated terms show how downstream dynamics, cross-time accumulation, and the selected feature jointly shape the response. Here, \emph{output visibility} measures the fraction of path cost reflected in output-distribution change for small perturbations, and \emph{feature--score alignment} measures the fraction of local output information captured by the chosen feature mean. In particular, a perturbation can change the output distribution while leaving the mean of the chosen feature unchanged to first order in perturbation strength.

\section{Controlled Perturbations Propagation with Equal Path Cost}
\label{sec:generative_response}

In this section, we test the theoretical framework by using controlled perturbations at equal path cost, which provides a common measure of perturbation strength for comparison.

\subsection{Experimental setup}
\label{sec:response_setup}

\begin{figure}[t]
 \centering
 \includegraphics[width=\linewidth]{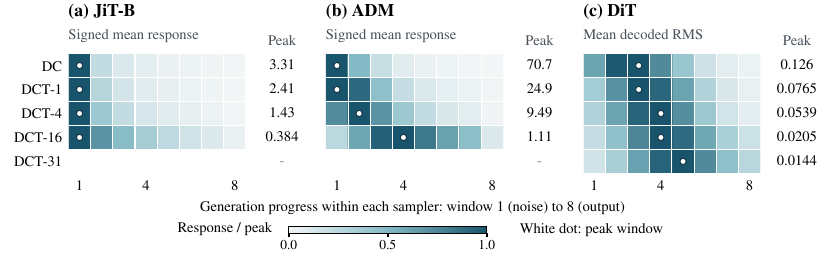}
 \caption{\textbf{Output responses at equal path cost within each model.}
 Rows show perturbation frequencies and columns show the sampling windows in which perturbations are applied, from early to late. Each row is divided by its maximum response. White dots mark the peak windows and numbers on the right give the maxima before normalization.}
 \label{fig:time_frequency}
\end{figure}

We study controlled perturbations in class-conditional ImageNet-256 generation using pixel-space JiT-B/L \citep{li2025jit} and ADM \citep{dhariwal2021diffusion}, and latent-space DiT \citep{peebles2023scalable}, without additional guidance. We group the 128 stochastic sampling steps into eight consecutive windows of 16 steps, numbered 1--8 in generation order from noise toward the output. In each experiment, we perturb the transition means within one window. To vary the spatial direction, we use unit-norm discrete cosine transform (DCT) templates, which represent spatial cosine patterns on the pixel or latent grid: DC is constant, and DCT-$k$ has higher frequency as $k$ increases. We normalize each perturbation to satisfy $\sum_i u_i^\top\Lambda_i^{-1}u_i=1$, giving equal path cost $\mathcal C(\varepsilon)=\varepsilon^2/2$ across windows and directions within each model, conditional on class and template. We use $\varepsilon=0.25$ for JiT/ADM and $0.8$ for DiT. For JiT and ADM, we measure changes in the mean DCT coefficients of generated images. For DiT, we measure the RMS image response after decoding. Measurement details are provided in Appendix~\ref{app:response_protocol}.

\subsection{Spatiotemporal sensitivity to perturbations}
\label{sec:time_frequency}

Figure~\ref{fig:time_frequency} shows how output responses vary with perturbation timing and spatial frequency at equal path cost. In JiT-B, all four frequencies shown produce their largest response in the first window, with responses decreasing toward later windows. In ADM, the peak shifts from window 1 for DC to window 4 for DCT-16; in DiT, it shifts from window 3 for DC to window 5 for DCT-31. Thus, earlier perturbations do not always produce larger output responses: the most sensitive stage depends on both the model and perturbation frequency. These observations are consistent with Section~\ref{sec:response_propagation}, where downstream propagation shapes the output effect of a perturbation beyond its path cost.

We provide two DiT image examples in Appendix~\ref{app:dit_visuals}, illustrating how perturbations with equal path cost produce different output changes when applied at different sampling stages.

\subsection{Local predictions and finite perturbation responses}
\label{sec:local_finite_response}

\begin{figure}[t]
 \centering
 \includegraphics[width=\linewidth]{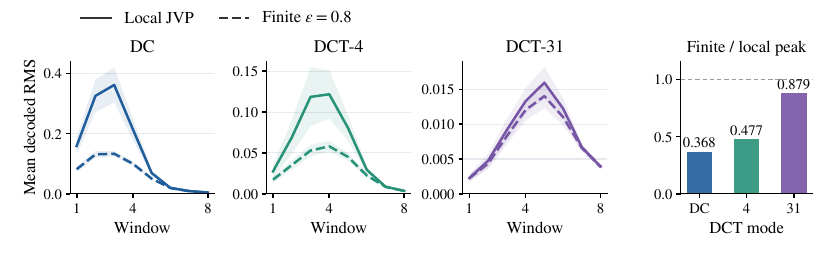}
 \caption{\textbf{Local predictions and finite perturbation responses in DiT.}
 Mean decoded-image RMS of local derivatives and symmetric finite differences ($\varepsilon=0.8$) on the same 32 trajectories, where shaded bands indicate one standard error of the mean across samples. The right panel reports finite/local peak ratios with dashed line denoting equal amplitudes. All runs use 128 stochastic steps.}
 \label{fig:dit_local_finite}
\end{figure}

The responses above are measured at finite perturbation strength, whereas the propagation analysis in Section~\ref{sec:response_propagation} describes the output response to infinitesimal perturbations. We use DiT to test how well this local prediction captures the magnitude and temporal pattern of responses to finite perturbations. On the same 32 sampling trajectories, we compare the mean per-sample RMS of the output derivative $V$, computed through the sampler and decoder using Jacobian--vector products (JVPs), with that of the symmetric finite difference $(Y^{+\varepsilon}-Y^{-\varepsilon})/(2\varepsilon)$ at $\varepsilon=0.8$.

As shown in Figure~\ref{fig:dit_local_finite}, the local and finite responses follow similar trends across sampling windows and peak at windows 3, 4, and 5 for DC, DCT-4, and DCT-31, respectively. However, the finite-response peaks are only 36.8\%, 47.7\%, and 87.9\% of the corresponding local predictions, as shown in the right panel. Thus, in this DiT example, local propagation identifies the most responsive sampling windows, while overestimating the peak response magnitudes at the tested finite amplitude.

\subsection{Feature responses relative to baseline fluctuations}
\label{sec:fluctuation_response}

Having examined propagation through the decoded-image response in DiT, we now return to the feature mean responses in JiT and ADM to compare them on their baseline fluctuation scales. Section~\ref{sec:observable_response} relates a feature's mean response to output information after normalization by its baseline variance. Table~\ref{tab:feature_responses} compares absolute responses, baseline fluctuations, and normalized responses at the first sampling window. In JiT-B/L, DCT-16 has a smaller absolute response than DC, but its baseline fluctuations are sufficiently smaller to give a larger normalized response. In ADM, the absolute response decreases more sharply than the baseline fluctuations, so DCT-16 also has a smaller normalized response. These results show that a smaller mean response can represent a larger change relative to a feature’s baseline fluctuations, consistent with the analysis in Section~\ref{sec:observable_response}. We provide the results for all sampling windows in Appendix~\ref{app:full_response_curves}.

\begin{table}[t]
 \centering\small
 \setlength{\tabcolsep}{3pt}
 \caption{\textbf{Feature responses relative to baseline fluctuations.}
 Perturbations act at window 1 at equal path cost, with $\varepsilon=0.25$ and $N=512$ for JiT-B/L and $256$ for ADM. We report absolute mean response $|\widehat\chi_f|$, baseline standard deviation $s_f$, and normalized response $\widehat B_f=\widehat\chi_f^2/s_f^2$.}
 \label{tab:feature_responses}
 \begin{tabular}{lrrr@{\hspace{10pt}}rrr@{\hspace{10pt}}rrr}
\toprule
 & \multicolumn{3}{c}{JiT-B} & \multicolumn{3}{c}{JiT-L} & \multicolumn{3}{c}{ADM}\\
\cmidrule(lr){2-4}\cmidrule(lr){5-7}\cmidrule(lr){8-10}
Feature & $|\widehat\chi_f|$ & $s_f$ & $\widehat B_f$ & $|\widehat\chi_f|$ & $s_f$ & $\widehat B_f$ & $|\widehat\chi_f|$ & $s_f$ & $\widehat B_f$\\
\midrule
DC & $3.31$ & $127$ & $0.000677$ & $3.36$ & $114$ & $0.000875$ & $70.7$ & $106$ & $0.447$\\
DCT-1 & $2.41$ & $48.5$ & $0.00248$ & $2.46$ & $45.2$ & $0.00296$ & $24.9$ & $42.5$ & $0.343$\\
DCT-4 & $1.43$ & $18.9$ & $0.00577$ & $1.46$ & $21.4$ & $0.00465$ & $6.9$ & $23.5$ & $0.0865$\\
DCT-16 & $0.384$ & $3.88$ & $0.00982$ & $0.366$ & $4.34$ & $0.00712$ & $0.274$ & $4.32$ & $0.00402$\\
\bottomrule
\end{tabular}

\end{table}

\section{Diagnosing Output Errors in Cache-Based Acceleration}
\label{sec:accumulated_defects}

In this section, we investigate whether the perturbation propagation analysis in Section \ref{sec:path_space} can predict output errors caused by practical computational approximations during diffusion sampling. 

We consider cache-based acceleration as a practical example during diffusion sampling, since reusing earlier model predictions reduces the number of model evaluations but introduces approximation errors into subsequent updates \citep{ma2024deepcache,liu2026epacache}. We divide 64 sampling steps into eight windows. Each candidate caches the prediction at the first step of one window $B$ and reuses it for the remaining seven steps, while other steps use fresh predictions. The differences between cached and fresh predictions induce local update perturbations $e_i$ along the reference trajectory, whose detailed derivation is provided in Appendix~\ref{app:cache_protocol}.

Using their first-order propagation, we construct a predictor of the resulting image mean squared error (MSE) and compare its predictions with observed errors:
\begin{equation}
 R_{\mathrm{pred}}=\frac{1}{d}\left\|\sum_{i\in B}T_i e_i\right\|^2,
 \qquad R_{\mathrm{obs}}=\frac{\|Y^{\mathrm{cache}}-Y\|^2}{d},
 \label{eq:cache_risk}
\end{equation}
where $d$ is the number of image entries and $Y^{\mathrm{cache}}$ is the final output when predictions are cached within $B$.

\begin{figure}[t]
 \centering
 \includegraphics[width=\linewidth]{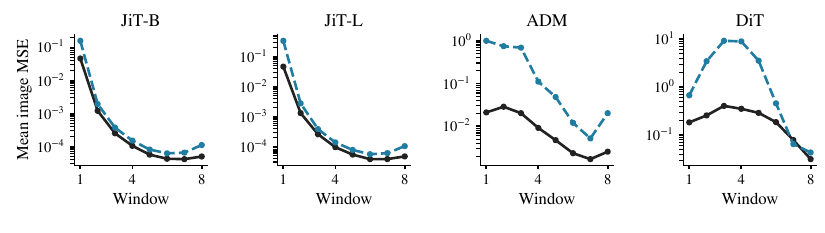}
 \caption{\textbf{Predicted and observed image MSE under caching.} Black solid lines show observed image MSE and blue dashed lines show predictions from the propagation-based predictor, averaged over the same 64 samples per model. Each candidate reuses the prediction cached at the start of one eight-step window, and windows are ordered from noise toward output.}
 \label{fig:cache_window_risks}
\end{figure}

\begin{table}[t]
 \caption{Spearman correlations with observed mean image MSE across eight caching windows, using the runs in Figure~\ref{fig:cache_window_risks}. Bold values mark the highest correlation for each model.}
 \label{tab:cache_ranking}
 \centering
 \setlength{\tabcolsep}{10pt}
 \renewcommand{\arraystretch}{1.08}
 \begin{tabular}{lccc}
\toprule
Model & Our predictor & Local defect & Path cost\\
\midrule
JiT-B & \textbf{0.9524} & 0.2143 & -0.4762\\
JiT-L & \textbf{0.9762} & 0.2381 & -0.4286\\
ADM & \textbf{0.9762} & -0.9048 & -0.9048\\
DiT & \textbf{0.9762} & 0.5000 & -0.3095\\
\bottomrule
\end{tabular}

\end{table}

Figure~\ref{fig:cache_window_risks} compares predicted and observed image MSE. The propagation-based predictor captures the broad variation in error across windows, although predicted and observed magnitudes can differ substantially. For JiT-B/L, the predictions also track the observed error magnitudes more closely, particularly in the middle and later windows.

In Table~\ref{tab:cache_ranking}, our propagation-based predictor achieves Spearman correlations of $0.9524$--$0.9762$ with observed mean image MSE, substantially exceeding the local-defect and path-cost baselines defined in Appendix~\ref{app:cache_predictors} across all four models. These results show that our propagation-based predictor can identify sampling windows in which caching causes larger output errors, which validated the analysis in Section \ref{sec:path_space}.

\section{Conclusion}
\label{sec:discussion}

In this work, we develop a theoretical framework explaining why path cost alone does not determine output effects in diffusion sampling. We separate the changes in output distribution from conditional trajectory changes and characterize the fraction of local output information captured by a selected feature mean. We express output responses as sums of propagated local perturbations, revealing how downstream dynamics, cross-time accumulation, and observation jointly shape measured effects. We use this theory to diagnose temporal and frequency-dependent sensitivity under equal path costs and identify caching windows with larger image deviations in the evaluated models.

\clearpage
\bibliography{paper}
\bibliographystyle{plainnat}
\clearpage
\appendix
\section{Proofs and derivations}
\label{app:path_information}

We follow the order of the theoretical framework in Section~\ref{sec:path_space}. We first derive the exact path cost and its decomposition into output and conditional-trajectory changes. We then prove the propagation identity, establish the local Fisher decomposition, and derive the feature-response bound and the main response identity. Throughout, $P$ and $Q_\varepsilon$ denote the joint distributions of reference and perturbed sampling trajectories, respectively.

\subsection{Derivation of the path cost for local perturbations}
\label{app:discrete_cost}

We derive how local changes in the sampling updates determine the distinguishability of complete trajectories. Starting from the Gaussian sampling rule in Eq.~\eqref{eq:reference_sampler}, we consider perturbations that shift the conditional transition mean while preserving the covariance function and initial distribution. We first express the trajectory likelihood ratio through the sampling transitions, then evaluate each transition's contribution to obtain Eq.~\eqref{eq:path_cost}.

\paragraph{Reference and perturbed transitions.}
We write $H_i=(X_0,\ldots,X_i)$ for the sampling history available before update $i$. Common conditioning variables, when present, are included in $X_0$ and remain fixed during generation. To allow a local perturbation to depend on previous sampling states, we write its mean displacement as $\varepsilon u_i(H_i)$. The reference and perturbed conditional transition distributions are
\begin{equation}
 \begin{aligned}
 K_i^0(\cdot\mid H_i)
 &=\mathcal N\!\left(\mu_i(X_i),\Lambda_i(X_i)\right),\\
 K_i^\varepsilon(\cdot\mid H_i)
 &=\mathcal N\!\left(\mu_i(X_i)+\varepsilon u_i(H_i),\Lambda_i(X_i)\right).
 \end{aligned}
 \label{eq:app_mean_kernels}
\end{equation}
Here $u_i=0$ at unperturbed steps. We assume $\Lambda_i(X_i)$ is positive definite, so both conditional densities are strictly positive. Together with the shared initial distribution, this ensures $Q_\varepsilon\ll P$ \citep{kallenberg2021foundations}. We take the perturbations to have finite expected quadratic cost under $Q_\varepsilon$, so the path KL derived below is finite.

\paragraph{From trajectory likelihoods to transition costs.}
Let $k_i^0$ and $k_i^\varepsilon$ denote the densities of the two kernels in Eq.~\eqref{eq:app_mean_kernels}. Recall that \(\Gamma = (X_0,\ldots,X_M)\) denotes a complete sampling trajectory. We factor the trajectory likelihood ratio into conditional density ratios:
\begin{equation}
 L_\varepsilon(\Gamma):=\frac{\dd Q_\varepsilon}{\dd P}(\Gamma)
 =\prod_{i=0}^{M-1}\ell_i^\varepsilon(X_{i+1},H_i),
 \qquad
 \ell_i^\varepsilon(x,H_i):=
 \frac{k_i^\varepsilon(x\mid H_i)}{k_i^0(x\mid H_i)}.
 \label{eq:app_path_likelihood_product}
\end{equation}
The common initial distribution cancels from this ratio. We take logarithms, average under $Q_\varepsilon$, and condition each term on $H_i$ to obtain
\begin{equation}
 \begin{aligned}
 \mathcal C(\varepsilon)
 &=\E_{Q_\varepsilon}\log L_\varepsilon(\Gamma)\\
 &=\sum_{i=0}^{M-1}\E_{Q_\varepsilon}
   \left[\E_{Q_\varepsilon}
   \left[\log\ell_i^\varepsilon(X_{i+1},H_i)\mid H_i\right]\right]\\
 &=\sum_{i=0}^{M-1}\E_{Q_\varepsilon}
 \KL\!\left(K_i^\varepsilon(\cdot\mid H_i)
             \middle\|K_i^0(\cdot\mid H_i)\right).
 \end{aligned}
 \label{eq:path_chain_budget}
\end{equation}
This follows from the relative-entropy chain rule \citep{cover2006elements}. Each conditional KL compares the two possible updates at the same history; the outer expectation averages over histories reached by the perturbed process.
%The factorization also applies to general absolutely continuous conditional kernels, so it does not require the perturbed state process to be Markov.

\paragraph{Evaluating the Gaussian contribution.}
We now compute each conditional KL explicitly. At a fixed history, we abbreviate $\mu_i=\mu_i(X_i)$, $\Lambda_i=\Lambda_i(X_i)$, and $u_i=u_i(H_i)$, and define the residual $r_i=X_{i+1}-\mu_i$. Since the two Gaussian kernels share their covariance, their density normalization factors cancel. Expanding their quadratic exponents gives
\begin{equation}
 \begin{aligned}
 \log\ell_i^\varepsilon(X_{i+1},H_i)
 &=-\frac12(r_i-\varepsilon u_i)^\top\Lambda_i^{-1}
                  (r_i-\varepsilon u_i)
   +\frac12r_i^\top\Lambda_i^{-1}r_i\\
 &=\varepsilon u_i^\top\Lambda_i^{-1}r_i
   -\frac{\varepsilon^2}{2}u_i^\top\Lambda_i^{-1}u_i.
 \end{aligned}
 \label{eq:app_gaussian_log_ratio}
\end{equation}
Under the perturbed transition, $\E_{Q_\varepsilon}[r_i\mid H_i]=\varepsilon u_i$. Taking the conditional expectation therefore yields
\begin{equation}
 \KL\!\left(K_i^\varepsilon(\cdot\mid H_i)
             \middle\|K_i^0(\cdot\mid H_i)\right)
 =\frac{\varepsilon^2}{2}u_i^\top\Lambda_i^{-1}u_i.
 \label{eq:app_gaussian_transition_cost}
\end{equation}
Substituting into Eq.~\eqref{eq:path_chain_budget}, we obtain the exact expression stated in Eq.~\eqref{eq:path_cost}:
\begin{equation}
 \mathcal C(\varepsilon)=\KL(Q_\varepsilon\|P)
 =\frac{\varepsilon^2}{2}\E_{Q_\varepsilon}
   \sum_{i=0}^{M-1}u_i^\top\Lambda_i^{-1}u_i.
 \label{eq:discrete_cost}
\end{equation}
Thus the trajectory-level change is the expected sum of local mean displacements weighted by inverse transition covariance. For $\Lambda_i=\nu_i I$, each contribution is $\varepsilon^2\|u_i\|^2/(2\nu_i)$, so an equal displacement has greater cost at a lower noise level. This identity holds at finite perturbation strength. The expectation remains under $Q_\varepsilon$ as the perturbation also changes which histories the sampler visits.
%State-dependent covariance is allowed: the required equality is between the reference and perturbed covariance functions evaluated at the same history. History-dependent perturbations, including those arising from stored predictions, are covered by the same derivation.

\subsection{Proof of Proposition~\ref{prop:kl_split}: Decomposition of Path and Output Changes}
\label{app:path_decomposition}

We prove Proposition~\ref{prop:kl_split} by factoring the trajectory likelihood ratio into output and conditional-trajectory likelihood ratios \citep{leonard2014survey}. We assume that the trajectory space is standard Borel and that \(Y\) is a measurable observation. Suppose \(Q_\varepsilon \ll P\) has finite KL, with likelihood \(L_\varepsilon = \mathrm{d}Q_\varepsilon/\mathrm{d}P\). Under these assumptions, regular conditional laws exist.
For every bounded measurable $g$, we have
\begin{equation}
 \E_{q_{\varepsilon,Y}}g(Y)=\E_P[L_\varepsilon g(Y)]
 =\E_P[\E_P(L_\varepsilon\mid Y)g(Y)].
\end{equation}
Identifying the output likelihood ratio, we obtain
\begin{equation}
 \ell_\varepsilon(Y):=\frac{\dd q_{\varepsilon,Y}}{\dd p_Y}(Y)
 =\E_P[L_\varepsilon\mid Y].
 \label{eq:project_likelihood}
\end{equation}
% Thus observing only the output averages the full path likelihood over its compatible histories. By disintegration, $\dd Q_\varepsilon(\cdot\mid Y)/\dd P(\cdot\mid Y)=L_\varepsilon/\ell_\varepsilon(Y)$, $Q_\varepsilon$-almost surely. We split $\log L_\varepsilon$ into $\log\ell_\varepsilon$ and $\log(L_\varepsilon/\ell_\varepsilon)$ and take expectations under $Q_\varepsilon$. 
Thus, when only the output is observed, the path likelihood is averaged over all histories compatible with that output. Disintegrating \(Q_\varepsilon\) and \(P\) with respect to \(Y\) gives
\begin{equation}
\frac{\mathrm{d}Q_\varepsilon(\cdot\mid Y)}{\mathrm{d}P(\cdot\mid Y)}
=
\frac{L_\varepsilon}{\ell_\varepsilon(Y)}
\quad Q_\varepsilon\text{-a.s.}
\end{equation}
This proves Eq.~\eqref{eq:kl_split}: the first expectation is output KL, and the second is conditional-history KL averaged over the perturbed output law. Nonnegativity of both terms gives $0\le\mathcal C_Y\le\mathcal C$.

\subsection{Proof of Proposition~\ref{prop:transport_response}: Perturbation propagation}
\label{app:propagation}

We prove Proposition~\ref{prop:transport_response} by differentiating the discrete sampler on its actual sampling grid. Consider $X_{i+1}^\varepsilon=\mathcal G_i^\varepsilon(X_i^\varepsilon,\xi_i)$ with the initial state and innovations held fixed. Assume the step maps are differentiable in state and perturbation strength along the reference trajectories, and $\phi$ is differentiable at their endpoints, almost surely. For mean response, we additionally impose the integrability and interchange conditions in Appendix~\ref{app:response}. Piecewise differentiable maps are treated at their differentiable points.

\paragraph{Forward transport.}
At fixed reference state and perturbation strength, we define $J_i=D_x\mathcal G_i^0$ and $e_i=\left.\partial_\varepsilon\mathcal G_i^\varepsilon\right|_0$. Define $\delta X_i=\left.\partial_\varepsilon X_i^\varepsilon\right|_0$. Applying the chain rule gives
\begin{equation}
 \delta X_{i+1}=J_i\delta X_i+e_i,\qquad \delta X_0=0.
 \label{eq:tangent_recursion}
\end{equation}
The first term propagates state changes from earlier steps, while the second accounts for the local change in the current update. When $e_i=0$, an earlier state change can still propagate through $J_i$. Iterating gives
\begin{equation}
 \delta X_M=\sum_iJ_{M-1}\cdots J_{i+1}e_i.
\end{equation}
Differentiating $Y^\varepsilon=\phi(X_M^\varepsilon)$ then gives Eq.~\eqref{eq:full_response_transport}. For a scalar output feature, the first-order expansion of its mean is
\begin{equation}
 \E_{Q_\varepsilon}f(Y)=\E_Pf(Y)+\varepsilon\chi_f+o(\varepsilon),
 \qquad \chi_f=\left.\partial_\varepsilon\E_{Q_\varepsilon}f(Y)\right|_0.
 \label{eq:mean_response_definition}
\end{equation}
Under the stated interchange conditions, substituting the expression for $V$ into $\chi_f=\E_P[\nabla f(Y)^\top V]$ yields Eq.~\eqref{eq:tangent_response}. For mean controls, $\mathcal G_i^\varepsilon=\mathcal G_i^0+\varepsilon u_i$ and $e_i=u_i(X_i)$; the control's state derivative is multiplied by $\varepsilon$ and contributes no additional term at zero. We define the local defects induced by output caching in Appendix~\ref{app:cache_protocol} and propagate them using this reference recursion.

\paragraph{Cross-time accumulation and statistical visibility.}
Writing $v_i=T_i e_i$, we expand $\langle\sum_i v_i,\sum_j v_j\rangle$ to obtain Eq.~\eqref{eq:response_accumulation}. The off-diagonal terms retain alignment between transported changes from different steps. Unlike the conditional likelihood scores in Eq.~\eqref{eq:history_fisher_sum} in Appendix~\ref{app:projection} , these output vectors need not be orthogonal in expectation.
% To complete the connection to statistical visibility, we substitute Eq.~\eqref{eq:tangent_response} into Eq.~\eqref{eq:app_statistic_alignment} and multiply by $F_Y/F_{\mathrm{path}}$. This proves Eq.~\eqref{eq:response_visibility_alignment}. When $F_Y=0$, Proposition~\ref{prop:response} instead directly gives $\chi_f=0$; no correlation with a zero-variance score is needed.

\paragraph{Adjoint evaluation.}
For a scalar statistic, we can evaluate the same response by propagating its gradient backward. We define $a_M=D\phi(X_M)^\top\nabla f(Y)$ and $a_i=J_i^\top a_{i+1}$. Combining this recursion with the tangent recursion, we obtain
\begin{equation}
 a_{i+1}^\top\delta X_{i+1}-a_i^\top\delta X_i=a_{i+1}^\top e_i,
 \qquad
 \left.\partial_\varepsilon f(Y^\varepsilon)\right|_0
 =\sum_i a_{i+1}^\top e_i.
\end{equation}
We derive the second identity by telescoping the first. Each term pairs the local injection with the gradient of the final statistic transported back to that step. For $\mathcal G_i=\mu_i+L_i\xi_i$, both tangent and adjoint recursions include the state derivative of $L_i$. With a deterministic time-zero map $h_0$ and decoder $D$ (including latent scaling), we obtain the terminal adjoint $a_M=J_{h_0}^\top J_D^\top\nabla f(Y)$, with each Jacobian evaluated at its corresponding input. These terms ensure that the derivative describes the observed output, including the final transition and decoding.

\subsection{Path Fisher information and output visibility}
\label{app:projection}

We next derive the local version of the exact decomposition. Recall that
$
S=\partial_\epsilon \log\frac{\mathrm{d}Q_\epsilon}{\mathrm{d}P}\Big|_{\epsilon=0}
$
is the path likelihood score. We assume the path likelihood admits the expansion
\begin{equation}
 L_\varepsilon=1+\varepsilon S+r_\varepsilon,
 \qquad \|r_\varepsilon\|_{L^2(P)}=o(|\varepsilon|),
 \qquad S\in L^2(P).
 \label{eq:likelihood_regularity}
\end{equation}
This condition controls the likelihood remainder in mean square and suffices for the response identities below. To obtain second-order KL expansions, we additionally assume that, for some $\delta>0$ and $\varepsilon_0>0$,
\begin{equation}
 \sup_{0<|\varepsilon|<\varepsilon_0}
 \E_P\left|\frac{L_\varepsilon-1}{\varepsilon}\right|^{2+\delta}<\infty.
 \label{eq:likelihood_moments}
\end{equation}
Normalization gives $\E_PS=0$. Conditional expectation contracts the $L^2$ remainder in Eq.~\eqref{eq:likelihood_regularity}, so Eq.~\eqref{eq:project_likelihood} has local score $S_Y=\E_P[S\mid Y]$. Because $S-S_Y$ is orthogonal to every square-integrable function of $Y$, we obtain
\begin{equation}
 \E_PS^2=\E_PS_Y^2+\E_P(S-S_Y)^2
 =F_Y+\E_P\Var_P(S\mid Y).
\end{equation}
This proves Eq.~\eqref{eq:fisher_split}, where $S_Y$ captures the component of the path score resolvable from outputs, while its residual accounts for the unresolved component. For a coarsening $\bar Y=\psi(Y)$, the tower property gives $S_{\bar Y}=\E_P[S_Y\mid\bar Y]$. Using Jensen's inequality then yields $F_{\bar Y}\le F_Y$.

To connect these quantities to KL curvature, we set $x_\varepsilon=L_\varepsilon-1$. The function $k(x)=(1+x)\log(1+x)-x$, extended continuously at $x=-1$, satisfies $k(x)=x^2/2+o(x^2)$ near zero and $k(x)\le Cx^2$ for $x\ge-1$. Equations~\eqref{eq:likelihood_regularity}--\eqref{eq:likelihood_moments} make the normalized squared increments uniformly integrable. Using $\E_Px_\varepsilon=0$, we can therefore take expectations of the expansion to obtain
\begin{equation}
 \KL(Q_\varepsilon\|P)=\tfrac12\varepsilon^2F_{\mathrm{path}}+o(\varepsilon^2),
 \qquad
 \KL(q_{\varepsilon,Y}\|p_Y)=\tfrac12\varepsilon^2F_Y+o(\varepsilon^2),
\end{equation}
which is exactly desired form of Eq.~\eqref{eq:path_fisher}. 
% The observed likelihood increments inherit the moment bound by conditional Jensen. 
Since the output likelihood ratio is the conditional expectation of the path likelihood ratio given the output, conditional Jensen implies that the observed likelihood increments inherit the same moment bound as the path likelihood increments.
Subtracting these expansions in Eq.~\eqref{eq:kl_split}, we also obtain $\mathcal C_{\mathrm{bridge}\mid Y}(\varepsilon)=\frac12\varepsilon^2F_{\mathrm{bridge}\mid Y}+o(\varepsilon^2)$. Thus $F_Y/F_{\mathrm{path}}$ measures the locally visible fraction whenever $F_{\mathrm{path}}>0$.

\paragraph{Additivity and Gaussian mean perturbations.}
For differentiable conditional kernels with the same initial law, we define the conditional scores $s_i=\left.\partial_\varepsilon\log\ell_i^\varepsilon(X_{i+1},H_i)\right|_0$ and $F_i(H_i)=\E_P[s_i^2\mid H_i]$. Assume that the scores are square integrable and differentiation can be interchanged with conditional normalization. Equation~\eqref{eq:app_path_likelihood_product} then gives $S=\sum_i s_i$ and $\E_P[s_i\mid H_i]=0$. For $i<j$, $s_i$ is measurable with respect to $H_j$, so $\E_P[s_i s_j]=\E_P[s_i\E_P(s_j\mid H_j)]=0$. We therefore obtain
\begin{equation}
 F_{\mathrm{path}}=\E_PS^2
 =\sum_{i=0}^{M-1}\E_PF_i(H_i).
 \label{eq:history_fisher_sum}
\end{equation}
For Gaussian mean perturbations, differentiating Eq.~\eqref{eq:app_gaussian_log_ratio} gives $s_i=u_i^\top\Lambda_i^{-1}(X_{i+1}-\mu_i)$. Using the reference residual's conditional covariance $\Lambda_i$, we obtain $F_i(H_i)=u_i^\top\Lambda_i^{-1}u_i$ and hence
\begin{equation}
 F_{\mathrm{path}}=\E_P\sum_{i=0}^{M-1}u_i^\top\Lambda_i^{-1}u_i.
\end{equation}

\subsection{Proof of Proposition~\ref{prop:response}: Feature response and output information}
\label{app:response}

We prove Proposition~\ref{prop:response} by differentiating the expectation of a statistic. For $f(Y)\in L^2(P)$, Eq.~\eqref{eq:likelihood_regularity} and Cauchy--Schwarz justify differentiating $\E_P[f(Y)L_\varepsilon]$ at zero. Conditioning on $Y$ and using $\E_PS_Y=0$, we derive
\begin{equation}
 \chi_f=\E_P[f(Y)S]=\E_P[f(Y)S_Y]
 =\E_P[(f(Y)-\E_Pf(Y))S_Y].
\end{equation}
Applying Cauchy--Schwarz once more gives $\chi_f^2\le\Var_P(f(Y))F_Y$. Combining this with Eq.~\eqref{eq:fisher_split} proves the bound in Eq.~\eqref{eq:response_bound}.

\subsection{Proof of Theorem~\ref{thm:propagation_visibility}: Perturbation propagation and measured output response}
\label{app:response_visibility}
We use $\E_PS_Y=0$ and Eq.~\eqref{eq:response_bound} to obtain $\chi_f=\Cov_P(f(Y),S_Y)$. For $F_Y>0$, we have $\Var_P(S_Y)=F_Y$ and derive
\begin{equation}
 \frac{B_f}{F_Y}
 =\frac{\chi_f^2}{\Var_P(f(Y))F_Y}
 =\operatorname{Corr}_P(f(Y),S_Y)^2.
 \label{eq:app_statistic_alignment}
\end{equation}
Multiplying by $F_Y/F_{\mathrm{path}}$ and substituting the propagated response from Eq.~\eqref{eq:tangent_response} proves Eq.~\eqref{eq:response_visibility_alignment}. If $F_Y=0$, Proposition~\ref{prop:response} gives $\chi_f=B_f=0$, completing the proof.

Thus the response depends on the statistic's alignment with the output score. The centered unit-variance statistic $f(Y)=S_Y/\sqrt{F_Y}$ attains equality in the response bound. If $F_Y=0$, every square-integrable mean response vanishes. We consequently obtain
\begin{equation}
 F_Y=\sup_{\E_Pf(Y)=0,\;\E_Pf(Y)^2=1}\chi_f^2,
\end{equation}
whenever the class contains a unit-variance observable. A chosen statistic measures one projection of this response geometry; optimizing over all such statistics recovers the full local output information.

\paragraph{Fluctuation scale and the limits of a selected mean.}
Writing $\Delta_f(\varepsilon)=\E_{Q_\varepsilon}f(Y)-\E_Pf(Y)$, we have $\Delta_f(\varepsilon)=\varepsilon\chi_f+o(\varepsilon)$. Thus $\varepsilon^2B_f$ is the leading squared mean shift in units of the reference feature variance. For $F_{\mathrm{path}}>0$, we obtain
\[
 \lim_{\varepsilon\to0}
 \frac{\Delta_f(\varepsilon)^2}
 {2\Var_P(f(Y))\,\mathcal C(\varepsilon)}
 =\frac{B_f}{F_{\mathrm{path}}}.
\]
% Replacing $f$ by $af+b$ with $a\ne0$ multiplies both its squared response and variance by $a^2$, leaving $B_f$ unchanged. A smaller absolute response can therefore correspond to a larger $B_f$ when the reference variance is sufficiently smaller. The covariance identity also shows that output changes orthogonal to the centered feature have zero first-order mean response, even when $F_Y>0$. This concerns the feature's mean, rather than all changes in its distribution. 
Replacing $f$ by $af+b$ with $a\ne0$ multiplies both its squared response and variance by $a^2$, leaving $B_f$ unchanged. A smaller absolute response can therefore correspond to a larger $B_f$ when the reference variance is sufficiently smaller. The covariance identity also shows that output changes orthogonal to the centered feature have zero first-order mean response, even when $F_Y>0$. This concerns the feature's mean, rather than all changes in its distribution. When $F_Y=0$, Eq.~\eqref{eq:response_bound} directly gives $\chi_f=0$; Eq.~\eqref{eq:response_visibility_alignment} is not used because correlation with a zero-variance score is undefined. Predicting finite mean shifts by $\varepsilon\chi_f$ additionally requires the higher-order remainder to be negligible.

To connect this statistical derivative with a dynamical derivative, we realize $Y^\varepsilon$ using common initial randomness and innovations, with marginal law $q_{\varepsilon,Y}$. Assume $f$ is continuously differentiable, $\partial_\varepsilon Y^\varepsilon|_0=V$ exists almost surely, and $[f(Y^\varepsilon)-f(Y)]/\varepsilon$ is dominated by an integrable random variable near zero. We exchange differentiation and expectation and apply the chain rule to obtain
\begin{equation}
 \chi_f=\left.\partial_\varepsilon\E f(Y^\varepsilon)\right|_0
 =\E[\nabla f(Y)^\top V].
\end{equation}
This is the same distributional mean derivative as the likelihood calculation. Substituting the expansion of $V$ from Eq.~\eqref{eq:full_response_transport} yields Eq.~\eqref{eq:response_visibility_alignment}.

\section{Gaussian construction for equal-cost perturbations}
\label{app:toy_extensions}

We give the Gaussian reverse-diffusion construction used in Figure~\ref{fig:toy_visibility}, in which perturbations with equal path cost reinforce or cancel at the output. Our goal is to hold the noise-normalized intervention strength fixed while changing how its contributions align after propagation. We take $C_0=\operatorname{diag}(1/4,1)$ and the forward chain $Z_{s+1}=\sqrt\alpha Z_s+\sqrt{1-\alpha}\xi_s$, with $\alpha=0.9$, $M=24$, and independent standard Gaussian innovations. We obtain the forward covariances $C_s=\alpha^sC_0+(1-\alpha^s)I$. Conditioning the joint Gaussian law, we derive the exact reverse kernels in generation order, $X_i=Z_{M-i}$:
\begin{equation}
 X_{i+1}=A_iX_i+\Lambda_i^{1/2}\xi_i,\quad
 A_i=\sqrt\alpha C_{s-1}C_s^{-1},\quad
 \Lambda_i=C_{s-1}-A_iC_sA_i^\top,\quad s=M-i.
 \label{eq:gaussian_reverse_kernels}
\end{equation}
We start from $X_0\sim\mathcal N(0,C_M)$ and observe $Y=X_M$, whose reference covariance is $C_0$. We apply deterministic mean controls:
\begin{equation}
 X_{i+1}=A_iX_i+\Lambda_i^{1/2}\xi_i+\varepsilon u_i,
 \qquad u_i=\sqrt{(\Lambda_i)_{11}}w_i e_1,
 \qquad Y=X_M.
 \label{eq:toy_controlled_chain}
\end{equation}

To make the propagation explicit, we define
\begin{equation}
 H_i=A_{M-1}\cdots A_{i+1},\qquad
 h_i=(H_i)_{11}\sqrt{(\Lambda_i)_{11}},
\end{equation}
with $H_{M-1}=I$. Here $H_i$ transports the injected offset to the endpoint, and $h_i$ includes its transition-noise scale. Let $w = (w_0,\ldots,w_{M-1})^{\top} $ be the temporal control vector, and $h = (h_0,\ldots,h_{M-1})^{\top} \in \mathbb{R}^M$ collect the downstream gains \(h_i\). We obtain the controlled endpoint mean $m_Y=\varepsilon(h^\top w)e_1$, and its covariance remains $C_0$. Using Eq.~\eqref{eq:discrete_cost} and the Gaussian endpoint KL, we derive
\begin{equation}
 F_{\mathrm{path}}=\|w\|^2,\qquad
 F_Y=\frac{(h^\top w)^2}{(C_0)_{11}},\qquad
 \mathcal C=\tfrac12\varepsilon^2\|w\|^2,\qquad
 \mathcal C_Y=\frac{\varepsilon^2(h^\top w)^2}{2(C_0)_{11}}.
 \label{eq:toy_cost_response}
\end{equation}
We choose the reinforcing direction $w_+=h/\|h\|$. For the cancelling direction, we define $r_i=(w_+)_i$ for $i<12$ and $r_i=-(w_+)_i$ otherwise, and set
\begin{equation}
 w_-=\frac{r-w_+(w_+^\top r)}{\|r-w_+(w_+^\top r)\|}.
\end{equation}
Both directions have unit norm, and $h^\top w_-=0$. Thus the budget depends on the norm of the temporal control, while output change depends on its projection onto the downstream gains. The cancelling control consequently preserves the full endpoint Gaussian law at every amplitude while changing the intermediate path law. At $\varepsilon=2$, the endpoint/bridge KL values are $(1.957579,0.042421)$ for the reinforcing control and $(0,2)$ for the cancelling control, as shown in Figure~\ref{fig:toy_visibility}. The intermediate mean obeys $m_{i+1}=A_im_i+\varepsilon u_i$, with $m_0=0$.

% For the along-trajectory views in Figure~\ref{fig:toy_visibility_full}, let $G_r=A_{M-1}\cdots A_r$, with $G_M=I$. We propagate the mean recurrence to the endpoint to obtain
% \[
%  G_rm_r=\varepsilon\sum_{i<r}H_iu_i
%        =\varepsilon\sum_{i<r}h_iw_i e_1.
% \]
% This is the terminal mean shift when interventions are applied only before node $r$. Panel~a plots its first coordinate and the accumulated magnitudes $\varepsilon\sum_{i<r}|h_iw_i|$. Under shared initial randomness and innovations, the paired displacement obeys the same deterministic mean recurrence, so the signed curve also gives the exact paired terminal displacement for that truncated intervention.

% At node $r$, the reference and controlled state covariances both equal $C_{M-r}$. The path prefix and its final state consequently have KLs
% \[
%  \mathcal C_{0:r}=\frac{\varepsilon^2}{2}\sum_{i<r}w_i^2,
%  \qquad
%  \mathcal C_{X_r}=\frac12m_r^\top C_{M-r}^{-1}m_r,
%  \qquad
%  0\le\mathcal C_{X_r}\le\mathcal C_{0:r}.
% \]
% These are the dashed and solid curves in panel~b. Their difference is the conditional-history KL given $X_r$, by the same chain rule as Eq.~\eqref{eq:kl_split}. The prefix cost is nondecreasing because each increment is nonnegative; current-state KL depends on the accumulated displacement and may decrease. At $r=M$, these curves recover the terminal costs given above.

The observable $f(Y)=Y_1$ has $\chi_f=h^\top w$ and $\Var_P(f(Y))=1/4$, hence $B_f=F_Y$ in this Gaussian mean-shift family.

\section{Experimental protocols and measurements}
\label{app:response_protocol}

\subsection{Measuring output responses}
\label{app:measurement_map}

For each sample, we generate two outputs, $Y_n^+$ and $Y_n^-$, by applying perturbations of strengths $+\varepsilon$ and $-\varepsilon$. We keep the class, initial noise, noise drawn at every sampling step, and spatial template identical in both runs. We then divide the output difference by $2\varepsilon$ to estimate the response to the perturbation.

For JiT and ADM, we project each generated image onto the DCT template used for the perturbation and average the resulting response over $N$ samples:
\begin{equation}
 \widehat\chi_f=\frac1N\sum_{n=1}^N\frac{f(Y_n^+)-f(Y_n^-)}{2\varepsilon},
 \qquad f(Y)=\langle q_k,Y\rangle,
 \label{eq:measured_scalar_response}
\end{equation}
where $q_k$ is the unit-norm image template defined below. For DiT, we compute the root-mean-square response across the decoded image entries for each sample, then average over samples:
\begin{equation}
 \widehat R=\frac1N\sum_{n=1}^N
 \frac{\|Y_n^+-Y_n^-\|}{2\varepsilon\sqrt d},
 \qquad d=3\cdot256^2.
 \label{eq:measured_vector_response}
\end{equation}

For the frequency comparisons in Section~\ref{sec:fluctuation_response}, we estimate the variance $s_f^2$ of each DCT projection from unperturbed images, using the sample variance with denominator $N-1$ across classes and noise realizations. We compute $\widehat B_f=\widehat\chi_f^2/s_f^2$, using the same baseline variance for all injection windows of a given feature. We report standard errors of the response means by dividing the sample standard deviation of the per-sample responses by $\sqrt N$.

\subsection{Reference samplers and perturbation design}
\label{app:sampler_protocols}

\paragraph{JiT.}
The conditional clean predictor $D_\theta(z,t,c)$ defines drift
\begin{equation}
 b_t(z,c)=\frac{D_\theta(z,t,c)-z}{1-t}
            +\lambda_0\bigl(tD_\theta(z,t,c)-z\bigr),
 \qquad \sigma_t=\sqrt{2\lambda_0}(1-t).
\end{equation}
The parameter $\lambda_0$ scales the added diffusion and its matched score drift. We use $\lambda_0=0.1$, standard normal pixel initialization, and Euler transitions on $t_i=i\Delta t$, $\Delta t=t_*/128$, terminating at $t_*=0.95$. Their covariance is $\Lambda_i=\nu_iI$ with $\nu_i=2\lambda_0(1-t_i)^2\Delta t$. We obtain the generated output as $Y=D_\theta(Z_{128},t_*,c)$. This stochastic representation is induced by the Gaussian interpolation $Z_t=tX+(1-t)E$: with an exact clean predictor, adding score drift $\lambda_0(tD_\theta-z)$ and its matched diffusion preserves the interpolation marginals. With a learned predictor, the displayed kernel specifies the reference law used for the comparisons.

\paragraph{ADM.}
We use the class-conditional ImageNet-256 ADM U-Net. The noise-prediction channels give the clean prediction, clipped to $[-1,1]$. The original 1000-step linear schedule of the denoising diffusion probabilistic model (DDPM) is respaced to 129 retained times: 128 positive-variance transitions with fixed-small posterior covariance, followed by a deterministic time-zero transition. The resulting pixel image is $Y$. The variance-output channels are unused in this fixed-covariance reference.

\paragraph{DiT.}
DiT-XL/2 evolves a $4\times32\times32$ latent state. The 1000-step linear DDPM schedule is similarly respaced to 129 retained times. The network's noise and variance outputs specify the transition mean and diagonal learned-range covariance $L_iL_i^\top$, with clean-latent clipping disabled. After 128 stochastic transitions, the observed output is
\begin{equation}
 Y=D_{\mathrm{VAE}}\bigl(h_0(Z_{128})/0.18215\bigr),
\end{equation}
Here $h_0$ is the deterministic time-zero transition and $D_{\mathrm{VAE}}$ is the variational autoencoder (VAE) decoder. We use the unclamped decoded pixels for response measurements . We compute the output response through the sampler and decoder, including the variance head, $h_0$, and latent scaling. All three model families use conditional predictions without classifier or classifier-free guidance amplification.

\paragraph{Spatial templates and temporal normalization.}
On an $H\times H$ grid, let $\psi_k(r,s)$ be proportional to $\cos(\pi k(s+1/2)/H)$ and normalized to unit Frobenius norm. These are horizontal modes of the type-II discrete cosine transform: $k=0$ is the constant DC mode, and $k>0$ is labeled DCT-$k$. We use $k\in\{0,1,4,16\}$ for all models and additionally $k=31$ for DiT. Pixel controls use $q_k=(1,1,1)^\top\otimes\psi_k/\sqrt3$, with $H=256$. The same template defines $f_k(Y)=\langle q_k,Y\rangle$. For a window $B$ and isotropic transition covariance $\nu_iI$, both pixel implementations produce the constant mean offset
\begin{equation}
 u_i=c_Bq_k\,\mathbf1\{i\in B\},
 \qquad c_B=\left(\sum_{i\in B}\nu_i^{-1}\right)^{-1/2}.
 \label{eq:pixel_control_normalization}
\end{equation}
For JiT this is equivalently a constant drift perturbation $c_Bq_k/\Delta t$ within the window. DiT uses $H=32$, $q_k=r\otimes\psi_k$ with per-sample channel signs $r_c\in\{-1/2,1/2\}$, and
\begin{equation}
 u_i(Z_i)=\frac{L_i(Z_i)q_k}{\sqrt{|B|}}\,\mathbf1\{i\in B\}.
 \label{eq:latent_control_normalization}
\end{equation}
In both cases the sum of conditional mean-shift Fisher costs is exactly one, giving path KL $\varepsilon^2/2$ by Eq.~\eqref{eq:discrete_cost}. Windows index progress through each model's sampling schedule.

\subsection{Models and sampling configurations}
\label{app:run_details}

\begin{table}[htbp]
 \caption{Sample counts $N$ and perturbation amplitudes $\varepsilon$ for the controlled perturbation experiments in Section~\ref{sec:generative_response}. All use 128 stochastic steps, eight injection windows, and path Fisher $F_{\mathrm{path}}=1$.}
 \label{tab:response_protocol}
 \centering
 \begin{tabular}{lrr}
 \toprule
 Model & $N$ & $\varepsilon$\\
 \midrule
 JiT-B/16 & 512 & 0.25\\
 JiT-L/16 & 512 & 0.25\\
 ADM & 256 & 0.25\\
 DiT-XL/2 & 256 & 0.80\\
 \bottomrule
 \end{tabular}
\end{table}

We use publicly released class-conditional ImageNet-256 checkpoints for JiT-B/16 and JiT-L/16 \citep{li2025jit}, ADM \citep{dhariwal2021diffusion}, and DiT-XL/2 \citep{peebles2023scalable}. For JiT, we use exponential moving average (EMA) weights. Table~\ref{tab:response_protocol} lists the sample counts and perturbation amplitudes.

We use fixed class subsets, with one noise realization per sample; the reported variability therefore includes both class and noise variation. For comparisons across sampling windows, we keep the classes, initial states, and noise at every step fixed. The DiT local/finite comparison uses 32 paired samples.

\subsection{Caching protocol and local defects}
\label{app:cache_protocol}

We use the reference kernels and observations in Appendix~\ref{app:sampler_protocols}, with $M=64$ stochastic steps. JiT uses $\Delta t=0.95/64$; ADM and DiT use 65 retained diffusion times, comprising 64 stochastic transitions and the final deterministic transition. Each model has 64 paired samples and eight candidate windows $B=\{a,\ldots,a+7\}$. For each candidate, we store the network prediction at step $a$ and reuse it directly for the next seven steps, omitting seven network evaluations. Outside $B$, we evaluate fresh predictions. The final deterministic prediction and decoder are also evaluated afresh. We pair the cached and reference runs using the same class, initial state, and noise at every step.

To construct the propagation-based predictor, we evaluate the prediction differences along the reference trajectory. Let $o_i(X_i)$ denote the fresh prediction at step $i$. For $i\in B$, the prediction defect is $\Delta o_i=o_a(X_a)-o_i(X_i)$; outside $B$, it is zero. Writing $\mathcal H_i(x,o,\xi)$ for the sampling update with a supplied prediction $o$, we define its local response to this defect by
\begin{equation}
 e_i=D_o\mathcal H_i(X_i,o_i(X_i),\xi_i)[\Delta o_i].
 \label{eq:cache_local_defect}
\end{equation}
We evaluate this directional derivative at the reference state and prediction, holding the state and noise fixed. The reference Jacobian $J_i$ includes the state dependence of the fresh network prediction. We then propagate the local defects using
\begin{equation}
 \delta Z_{i+1}=J_i\delta Z_i+e_i,
 \qquad \delta Z_0=0,
 \qquad V=\delta Y=\sum_iT_i e_i.
\end{equation}
At the first step of the window, $e_a=0$. We compare $\|V\|^2/d$ with the observed image MSE from direct reuse, as in Eq.~\eqref{eq:cache_risk}. This predictor is a first-order approximation along the reference trajectory; its accuracy for the finite caching error is assessed empirically.

For JiT, $o_i$ is the clean prediction and the local update defect is $e_i=\Delta t\,[b_i(Z_i;o_a)-b_i(Z_i;o_i)]$. For ADM, $o_i$ is the clipped clean prediction, the posterior covariance is fixed, and $e_i=\mu_i(Z_i;o_a)-\mu_i(Z_i;o_i)$. Both means are affine in their supplied clean predictions, so these directional derivatives equal the update differences at the fixed reference state.

For DiT, we reuse the complete noise-and-variance prediction. Write $\Sigma_i=L_iL_i^\top$ for its diagonal covariance and $\ell_i=\log\operatorname{diag}\Sigma_i$. At the reference state and prediction, define $\dot\mu_i=D_o\mu_i[\Delta o_i]$ and $\dot\ell_i=D_o\ell_i[\Delta o_i]$. The reparameterized step has local response
\begin{equation}
 e_i=\dot\mu_i+\tfrac12L_i\operatorname{diag}(\dot\ell_i)\xi_i.
 \label{eq:cache_variance_defect}
\end{equation}
All state Jacobians include the learned covariance derivative. We propagate the response through the deterministic time-zero map, latent scaling, and VAE decoder. Thus the predicted and observed errors are both measured in decoded pixels.

\subsection{Caching error prediction and evaluation}
\label{app:cache_predictors}

For each sample and candidate window, we compute $R_{\mathrm{pred}}$ using the tangent recursion. The local-defect baseline is $L_{\mathrm{state}}=\sum_i\|e_i\|^2/d_z$, where $d_z$ is the pixel-state dimension for JiT/ADM and the latent-state dimension for DiT.

The path-cost baseline uses the local quadratic approximation $L_{\mathrm{path}}=\sum_i F_i/2$. At fixed covariance, $F_i=e_i^\top\Sigma_i^{-1}e_i$. For DiT's diagonal Gaussian transition,
\begin{equation}
 F_i=\sum_j\left[\frac{\dot\mu_{i,j}^{\,2}}{\Sigma_{i,jj}}+\frac12\dot\ell_{i,j}^{\,2}\right].
 \label{eq:cache_local_fisher}
\end{equation}
Here $F_i$ measures conditional Fisher information in the direction of the prediction defect, with the trajectory history held fixed. Its reference expectation summed over steps gives the corresponding path Fisher information. The resulting baseline is a local approximation to path cost, rather than an evaluation of the finite path KL under caching. We average each baseline over samples and compare its ranking across windows with that of observed mean image MSE.

\subsection{Image visualization protocol}
\label{app:dit_display}

The DiT examples use DCT-4 perturbations in windows 4 and 8 with $\varepsilon=0.8$, giving path KL $0.32$ in each window. Images are clipped to $[-1,1]$ for display; differences use unclipped decoded pixels. Difference maps show per-pixel RGB RMS relative to the reference, with a common color scale across samples and windows. The scale saturates at the pooled 99.5th percentile. Central-response annotations report the full-image RMS of $(Y^+-Y^-)/(2\varepsilon)$.

\section{Supplementary experimental results}
\label{app:results}

\subsection{DiT image examples at equal path cost}
\label{app:dit_visuals}

Figures~\ref{fig:dit_observations} and~\ref{fig:dit_second} show two fixed DiT samples under DCT-4 perturbations in windows 4 and 8. Each perturbation has path cost $\mathcal C=0.32$ at $\varepsilon=0.8$. The central-response RMS is $19.88\times$ and $11.23\times$ larger in window 4 than in window 8 for the two samples, respectively, illustrating the dependence of output response on perturbation timing at equal path cost. Both examples use 128 stochastic steps and the display protocol in Appendix~\ref{app:dit_display}.

\begin{figure}[!htbp]
 \centering
 \includegraphics[width=0.9\linewidth]{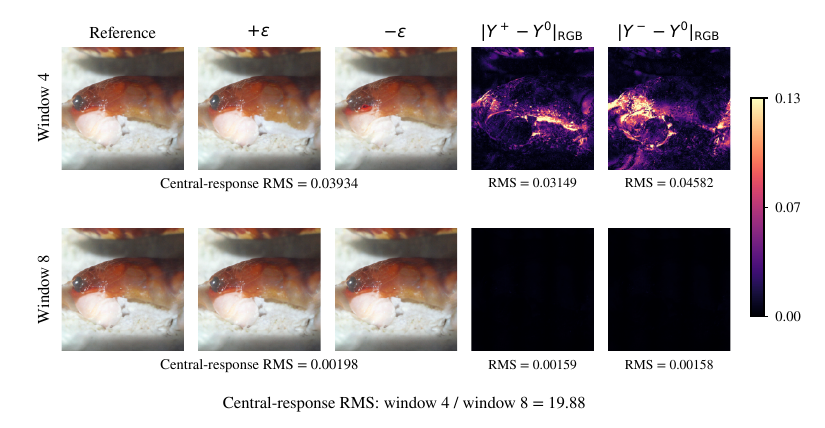}
 \caption{\textbf{DiT image responses to equal-cost perturbations.} A fixed sample under positive and negative DCT-4 perturbations in windows 4 and 8. Difference maps show per-pixel RGB RMS relative to the reference $Y^0$, with full-image RMS below each map. Central-response annotations measure the RMS of $(Y^+-Y^-)/(2\varepsilon)$; the window-4/window-8 ratio is $19.88$. Differences use unclipped pixels and a shared color scale across both examples.}
 \label{fig:dit_observations}
\end{figure}

\begin{figure}[!htbp]
 \centering
 \includegraphics[width=0.9\linewidth]{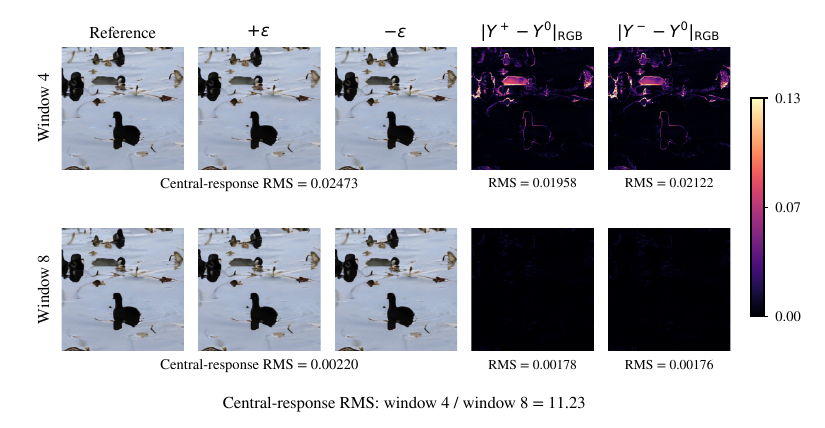}
 \caption{Second DiT example, with the endpoints, branch-difference maps, and shared color scale defined in Appendix~\ref{app:dit_display}. Reference-relative displacement RMS is annotated separately for each branch; the window-4/window-8 central-response RMS ratio is $11.23$.}
 \label{fig:dit_second}
\end{figure}

\FloatBarrier
\subsection{Feature responses across sampling windows}
\label{app:full_response_curves}

Tables~\ref{tab:feature_responses_full} and~\ref{tab:normalized_responses_full} extend Table~\ref{tab:feature_responses} to all eight sampling windows. For each feature, baseline variance is fixed across windows, so the temporal variation in $\widehat B_f$ comes from its mean response. Across features, the different baseline fluctuations determine how absolute responses translate into normalized responses.

\begin{table}[htbp]
 \centering\small
 \setlength{\tabcolsep}{3pt}
 \caption{\textbf{Feature mean responses across sampling windows.} Signed mean responses $\widehat\chi_f$ for windows W1--W8, ordered from early to late sampling, with baseline standard deviations $s_f$. All use $\varepsilon=0.25$, with $N=512$ for JiT-B/L and $256$ for ADM.}
 \label{tab:feature_responses_full}
 \begin{tabular}{llrrrrrrrrr}
\toprule
Model & Feature & $s_f$ & W1 & W2 & W3 & W4 & W5 & W6 & W7 & W8\\
\midrule
JiT-B & DC & $127$ & $3.31$ & $0.704$ & $0.361$ & $0.212$ & $0.128$ & $0.0786$ & $0.0462$ & $0.0187$\\
 & DCT-1 & $48.5$ & $2.41$ & $0.661$ & $0.345$ & $0.206$ & $0.132$ & $0.0785$ & $0.0442$ & $0.0131$\\
 & DCT-4 & $18.9$ & $1.43$ & $0.538$ & $0.303$ & $0.189$ & $0.12$ & $0.0768$ & $0.0407$ & $0.0164$\\
 & DCT-16 & $3.88$ & $0.384$ & $0.264$ & $0.183$ & $0.131$ & $0.0908$ & $0.061$ & $0.0371$ & $0.0149$\\
\midrule
JiT-L & DC & $114$ & $3.36$ & $0.709$ & $0.358$ & $0.215$ & $0.132$ & $0.0783$ & $0.0424$ & $0.0165$\\
 & DCT-1 & $45.2$ & $2.46$ & $0.67$ & $0.351$ & $0.21$ & $0.129$ & $0.0821$ & $0.0432$ & $0.0168$\\
 & DCT-4 & $21.4$ & $1.46$ & $0.548$ & $0.306$ & $0.191$ & $0.123$ & $0.0773$ & $0.0443$ & $0.0157$\\
 & DCT-16 & $4.34$ & $0.366$ & $0.259$ & $0.183$ & $0.131$ & $0.0925$ & $0.062$ & $0.0377$ & $0.0151$\\
\midrule
ADM & DC & $106$ & $70.7$ & $35.1$ & $13.8$ & $5.74$ & $2.71$ & $1.42$ & $0.782$ & $0.14$\\
 & DCT-1 & $42.5$ & $24.9$ & $22.4$ & $11$ & $5.07$ & $2.53$ & $1.36$ & $0.764$ & $0.138$\\
 & DCT-4 & $23.5$ & $6.9$ & $9.49$ & $6.53$ & $3.62$ & $2.03$ & $1.19$ & $0.706$ & $0.135$\\
 & DCT-16 & $4.32$ & $0.274$ & $0.662$ & $1.04$ & $1.11$ & $0.912$ & $0.686$ & $0.501$ & $0.118$\\
\bottomrule
\end{tabular}

\end{table}

\begin{table}[htbp]
 \centering\small
 \setlength{\tabcolsep}{3pt}
 \caption{\textbf{Feature responses on the baseline fluctuation scale.} Normalized responses $\widehat B_f=\widehat\chi_f^2/s_f^2$ for the measurements in Table~\ref{tab:feature_responses_full}.}
 \label{tab:normalized_responses_full}
 \begin{tabular}{llrrrrrrrr}
\toprule
Model & Feature & W1 & W2 & W3 & W4 & W5 & W6 & W7 & W8\\
\midrule
JiT-B & DC & $0.000677$ & $3.06\mathrm{e}{-5}$ & $8.05\mathrm{e}{-6}$ & $2.77\mathrm{e}{-6}$ & $1.01\mathrm{e}{-6}$ & $3.82\mathrm{e}{-7}$ & $1.32\mathrm{e}{-7}$ & $2.17\mathrm{e}{-8}$\\
 & DCT-1 & $0.00248$ & $0.000186$ & $5.07\mathrm{e}{-5}$ & $1.81\mathrm{e}{-5}$ & $7.38\mathrm{e}{-6}$ & $2.62\mathrm{e}{-6}$ & $8.32\mathrm{e}{-7}$ & $7.26\mathrm{e}{-8}$\\
 & DCT-4 & $0.00577$ & $0.000813$ & $0.000258$ & $0.000101$ & $4.08\mathrm{e}{-5}$ & $1.66\mathrm{e}{-5}$ & $4.66\mathrm{e}{-6}$ & $7.52\mathrm{e}{-7}$\\
 & DCT-16 & $0.00982$ & $0.00465$ & $0.00222$ & $0.00114$ & $0.000548$ & $0.000247$ & $9.13\mathrm{e}{-5}$ & $1.48\mathrm{e}{-5}$\\
\midrule
JiT-L & DC & $0.000875$ & $3.89\mathrm{e}{-5}$ & $9.92\mathrm{e}{-6}$ & $3.56\mathrm{e}{-6}$ & $1.35\mathrm{e}{-6}$ & $4.73\mathrm{e}{-7}$ & $1.39\mathrm{e}{-7}$ & $2.1\mathrm{e}{-8}$\\
 & DCT-1 & $0.00296$ & $0.00022$ & $6.03\mathrm{e}{-5}$ & $2.16\mathrm{e}{-5}$ & $8.2\mathrm{e}{-6}$ & $3.3\mathrm{e}{-6}$ & $9.14\mathrm{e}{-7}$ & $1.39\mathrm{e}{-7}$\\
 & DCT-4 & $0.00465$ & $0.000653$ & $0.000204$ & $7.98\mathrm{e}{-5}$ & $3.3\mathrm{e}{-5}$ & $1.3\mathrm{e}{-5}$ & $4.28\mathrm{e}{-6}$ & $5.34\mathrm{e}{-7}$\\
 & DCT-16 & $0.00712$ & $0.00356$ & $0.00177$ & $0.000912$ & $0.000454$ & $0.000204$ & $7.57\mathrm{e}{-5}$ & $1.21\mathrm{e}{-5}$\\
\midrule
ADM & DC & $0.447$ & $0.11$ & $0.017$ & $0.00295$ & $0.000655$ & $0.00018$ & $5.47\mathrm{e}{-5}$ & $1.75\mathrm{e}{-6}$\\
 & DCT-1 & $0.343$ & $0.277$ & $0.0668$ & $0.0143$ & $0.00354$ & $0.00103$ & $0.000324$ & $1.06\mathrm{e}{-5}$\\
 & DCT-4 & $0.0865$ & $0.164$ & $0.0774$ & $0.0239$ & $0.00749$ & $0.00255$ & $0.000906$ & $3.3\mathrm{e}{-5}$\\
 & DCT-16 & $0.00402$ & $0.0235$ & $0.0585$ & $0.0661$ & $0.0446$ & $0.0252$ & $0.0134$ & $0.00074$\\
\bottomrule
\end{tabular}

\end{table}

\end{document}